\documentclass[11pt]{article}%
\usepackage{mathrsfs}
\usepackage{latexsym}
\usepackage{amsmath}
\usepackage{amssymb}
\usepackage{amsfonts}
\usepackage{graphicx}
\usepackage{float}
\usepackage{amsthm}
\usepackage{epstopdf}
\usepackage{booktabs}
\usepackage{lscape}
\usepackage{chngpage}
\usepackage{array}
\usepackage{color}
\usepackage{textcomp}
\usepackage{cite}
\usepackage{subfigure}
\usepackage{bbding}%
\usepackage{algorithmic}
\usepackage{algorithm}
\usepackage{graphicx}
\usepackage{textcomp}
\usepackage{setspace}
\usepackage{epsfig,float}
\usepackage{multirow}
\usepackage{booktabs}
\usepackage{algorithm}
\usepackage{color} 
\providecommand{\U}[1]{\protect\rule{.1in}{.1in}}
\providecommand{\U}[1]{\protect\rule{.1in}{.1in}}
\providecommand{\U}[1]{\protect\rule{.1in}{.1in}}
\RequirePackage[colorlinks,citecolor=blue,urlcolor=blue]{hyperref}
\providecommand{\U}[1]{\protect\rule{.1in}{.1in}}
\numberwithin{equation}{section}
\newtheorem{theorem}{Theorem}[section]

\graphicspath{{figures/}}

\begin{document}
%
%\author{
%    Author1
%    \footnotemark[2],
%    Author2
%    \footnotemark[2]
%    \footnotemark[3],
%    Author3\footnotemark[4],
%    Author4\footnotemark[4]}
%
%\renewcommand{\thefootnote}{\fnsymbol{footnote}}
%
%
%
%\footnotetext[2]{Address of Author1}
%
%\footnotetext[3]{Address ofAuthor2}
%
%\footnotetext[4]{Address of Author3 and Author4}

\title{\vspace{-1in}\parbox{\linewidth}{
} \vspace{\bigskipamount} \\A Group Norm Regularized Factorization Model for Subspace Segmentation \footnotemark[1] }
\author{Xishun Wang\footnotemark[4] \footnotemark[2],\ \ Zhouwang Yang\footnotemark[3] \ \ Xingye Yue\footnotemark[2], \ \ Hui Wang\footnotemark[5]}
\date{\today}
\maketitle
%Email:yuhong.xu@hotmail.com
\vspace{-6mm}
\renewcommand{\thefootnote}{\fnsymbol{footnote}}
\footnotetext[1]{This work was supported by the National Natural Science Foundation of China under Grant 11971342.}
\footnotetext[2]{Center for Financial Engineering, Soochow University, Suzhou 215006, China.}

\footnotetext[3]{School of Mathematical Sciences, University of Science and Technology of China, Hefei 230026, China.}
\footnotetext[5]{Big Data Service Department, State Grid Corporation Customer Service Center, Nanjing 210000, China.}

\footnotetext[4]{Corresponding author, xswang11@stu.suda.edu.cn}
\indent\textbf{Abstract }Subspace segmentation assumes that data comes from the union of different subspaces and the purpose of segmentation is to partition the data into the corresponding subspace. Low-rank representation~(LRR) is a classic spectral-type method for solving subspace segmentation problems, that is, one first obtains an affinity matrix by solving a LRR model and then performs spectral clustering for segmentation. This paper proposes a group norm regularized factorization model~(GNRFM) inspired by the LRR model for subspace segmentation and then designs an Accelerated Augmented Lagrangian Method~(AALM) algorithm to solve this model. Specifically, we adopt group norm regularization to make the columns of the factor matrix sparse, thereby achieving a purpose of low rank, which means no Singular Value Decompositions~(SVD) are required and the computational complexity of each step is greatly reduced. We obtain affinity matrices by using different LRR models and then performing cluster testing on different sets of synthetic noisy data and real data, respectively. Compared with traditional models and algorithms, the proposed method is faster and more robust to noise, so the final clustering results are better. Moreover, the numerical results show that our algorithm converges fast and only requires approximately ten iterations.
%Its performance
%in stressed markets shows that CGVaR could alleviate the fat-tail effect such
%that it even performs better than Gumbel CVaR.

\indent\textbf{Key words:} Low-rank representation, Group norm regularization, Subspace segmentation, Affinity matrix, Spectral clustering.

%\indent\textbf{AMS subject classifications:} 91B30, 62P05
%\indent\textbf{JEL classification:} C6

\section{Introduction}

With the advent of the era of big data, we are confronted with a great deal of data every day. Although the data volume is large, it may only come from several low-rank subspaces. Subspace segmentation\cite{liu2010robust} divides data into several clusters and each cluster is a subspace; this problem arises in machine learning, computer vision, image processing, finance, and other fields\cite{lu2006combined,rao2010motion,ma2007segmentation,Zhang2005Subspace,Deng2013Low}. Subspace segmentation is an important clustering problem, which is mainly solved by the following four methods: mixtrue of Gaussian\cite{yang2006robust}, factorization \cite{gruber2004multibody}, algebraic\cite{ma2008estimation}, and spectral-type methods \cite{Elhamifar2012Sparse,liu2010robust}.

In the spectral-type method, one first calculates an affinity matrix and then spectral clustering (such as Normalized Cuts~($N\_Cut$)~\cite{shi2000normalized} ) is performed.
The main variation among spectral-type methods is  the different affinity matrix learning methods. Spectral clustering has attractive advantages, for example the algorithm is efficient, the data can be of any shape, the method is not sensitive to abnormal data, and can be applied to high-dimensional problems. Recently there have been many studies and developments in spectral clustering, such as \cite{yang2017discrete,yang2015multitask,Yang2013Discriminative,wang2019an,wang2020active,ding2019a}.

LRR is a classic, effective spectral-type method for solving subspace segmentation problems. In 2010, the Low-Rank Representation (LRR)~\cite{liu2010robust} problem was proposed by Liu $et \ al$. They assume that data samples come from the union of multiple subspaces, and the purpose of the LRR method is to denoise and obtain samples from the corresponding subspaces to which they belong. And they proved that LRR accurately obtains each real subspace for clean data. For noisy data, LRR approximately restores the subspace of the original data with theoretical guarantees. In reference~\cite{liu2010robust}, the spectral clustering performance using the affinity matrix obtained by LRR is more accurate and robust than other methods.

When solving the LRR problem, traditional methods tend to minimize the nuclear norm to approximate the minimum rank in the objective function. This is a convex approximation that guarantees convergence in the designed algorithm. However, singular value decomposition~(SVD) is required in order to solve nuclear norm problems. An SVD is time- consuming, and the computation complexity is $O(n^3)$ for an $n\times n$ affinity matrix. Many classic algorithms employ SVD to solve LRR,  such as the Accelerated Proximal Gradient method (APG~\cite{toh2010accelerated}), Alternating Direction Method (ADM~\cite{lin2010augmented}),  and Linearized Alternating Direction Method with Adaptive Penalty (LADMAP~\cite{lin2011linearized}). APG solves an approximated problem of LRR, but the clustering results are inferior. LADMAP performs best among these algorithms; however,  the calculation speed is still slow, especially for high-dimensional data. Along this line of thinking, accelerated LADMAP~(LADMAP(A)~\cite{lin2011linearized}) is proposed by Lin $et \ al$. They used skinny SVD technology to reduce the complexity to $O(rn^2)$, where $r$ is the rank of the affinity matrix. However, the rate of convergence is sub-linear, which requires  more iterations, and the rank depends on hyperparameter selection. Lu $et \ al.$ introduced a smooth objective function with regularization terms and used the Iterative Reweighted Least Squares (IRLS) method to solve objective functions~\cite{lu2014smoothed}. This new method does not need SVD, but the computation complexity of matrix multiplication is $O(n^3)$. Their numerical experiments show that the convergence is linear, so it is faster than LADMAP(A) in some cases.

In order to avoid the need for SVD calculation, Chen et al. offered the matrix factorization LRR model and Hidden Matrix Factors Augmented Lagrangian Method (HMFALM)~\cite{chen2018algorithm}. They decomposed an affinity matrix into $UV$ and then used the Augmented Lagrangian Method~(ALM) to solve the model, where $U \in R^{n\times r}$, $V \in R^{r \times n}$. Factorization model is difficult to determine rank, so Chen et al. designed a greedy method to traverse the rank $r$, that is, they first choose a proper interval $d$ and run the algorithm on the ranks 1, $d+1$, $2d+1$,..., $k*d+1$,..., and stop when the results begin to worsen. Thus, this process searches through the options one by one to find the optimum rank. Although the original problem becomes  non-convex, the algorithm does not require SVD; only multiplication of the  factor  matrix  is  required.  Its  complexity  is $O(rmn)$, where $m$ is  the dimension of the data.

HMFALM requires an outer loop to find rank $r$, the inner loop iterates to meet the stopping criterion, and rank-finding result is heavily dependent on the given hyperparameter. To overcome the shortcomings of HMFALM, we introduce group norm regularization of $U$ to design an adaptive rank-finding matrix factorization model (GNRFM) to solve LRR. We first let $U \in R^{n\times K}$, where $K$ is a larger number. Group norm regularization $l_{2,1}$
makes some columns of the factor matrix~$U$ become zero vector columns. In this manner, the rank of the affinity matrix is automatically reduced to achieve the purpose of adaptively adjusting the rank. We design the Accelerated Augmented Lagrangian Method~(AALM) algorithm to solve GNRFM. In summary, the contributions of this work include:
\begin{itemize}
\item  We  first use the group norm regularization method to solve the rank minimum problem. Specifically, we design the GNRFM model to solve the subspace segmentation problem. Compared with the traditional nuclear norm LRR model,
    our model has less computational complexity. Compared with the factorization model HMFALM\cite{chen2018algorithm}, the GNRFM model can adaptively find ranks without greedy searches.
\item The group regularization term also has positive anti-noise effects, so the GNRFM is more robust.
\item We design the AALM algorithm to solve GNRFM, which uses the acceleration technique of one-step inner iteration and deletes the zero vector column to reduce the computational complexity in each step. The numerical results show that the AALM algorithm converges in about ten steps.
\end{itemize}

%Although we decreasing the rank  from a large number, the numerical results show that zero columns appear very quickly (zero columns can be deleted to speed up). It drops to a low rank in a few steps, and converges only with about ten iterations. Its  specific solving algorithm is ALM method, which is similar to \cite{chen2018algorithm}, and the algorithm complexity is $O(rmn)$. Numerical results on noisy synthetic data and real data~(Hopkin155 and EYaleB) show that our model is faster, more accurate, and more robust to noise than above algorithms.

The structure of this paper is as follows. Section~2 introduces the LRR problem,  its nuclear norm approximation model~\cite{liu2010robust}, and the matrix factorization model \cite{chen2018algorithm}. Section~3 introduces our model, GNRFM, details the Accelerated ALM (AALM) for our model, gives time complexity analysis and introduces how to use GNRFM's solution for spectral clustering. Numerical experimental results are reported in Section~4. Finally, Sections~5 concludes this paper.

\section{The LRR problem and two types of models}
In order to facilitate a shared understanding of notations, we offer a summary in Table 1 of the primary notations used in this paper.

\begin{table}[!htbp]
 \caption{ Summary of Common Notations }
 \footnotesize
\setlength{\tabcolsep}{5pt}
 \centering
  %  \normalsize
 \begin{tabular}{c|l}
  \toprule[1.5pt]
   $Z$& The solution to the LRR problem  \\
   $W$ & The affinity matrix \\
   $X$ & The data matrix \\
   $E$ & The LRR problem's error matrix  \\
   $U$ & The factor matrix of factorization models HMFALM and GNRFM \\
   $V$ & The factor matrix of factorization models HMFALM and GNRFM \\
   $m$ & The dimension of a data matrix \\
   $n$ & The number of data samples \\
   $r$ & The rank of Z (column U and row V)\\
   $K$ & The oversize rank estimate \\
   $\mu$ &The regularized parameter in the LRR and HMFALM models\\
   $\mu_U$ &The group norm regularized parameter in GNRFM\\
   $\mu_V$ &The Frobenius norm regularized parameter in GNRFM\\
   $Y$ & The multiplier matrix in the ALM and AALM algorithms\\
   $\beta$ & The penalty parameter in the ALM and AALM algorithms\\
   $t$ & The outer iteration step\\
   $l$ & The inner iteration step\\
   $\sigma$ & The noise intensity \\
   $\widetilde{k}$ & The number of clusters\\
   $||.||_F$ & The Frobenius norm\\
   $||.||_{2,1}$ & The group norm $l_{2,1}$\\
   $()_{\star i}$ & The i-th column of a matrix \\
   $()_{i \star }$ & The i-th row of a matrix \\
  \bottomrule[1.5pt]
 \end{tabular}
 \label{bs1}
\end{table}

First, let us recall the following LRR problem:
\begin{equation}\label{1}
\min_Z rank(Z), \ \   s.t. \ X=XZ,
\end{equation}
where $X \in R^{m \times n}$~is the data matrix, $m$ is the dimension of the data vector, $n$ is the number of data vectors,~and $Z \in R^{n \times n}$. We refer to the optimal solution $Z^\star$ of the above problem as the "lowest-rank representation" of data $X$ with respect to a dictionary $X$. This is an NP-hard problem because the rank is $l_0$ norm, and the solution is not unique. As in the classic method for solving low-rank problems, Liu et al.~\cite{liu2010robust} took advantage of the nuclear norm to approximate and obtain the following convex optimization problem:
\begin{equation}\label{2}
\min_Z ||Z||_\star, \ \  s.t. \ X=XZ.
\end{equation}

Liu et al.~\cite{liu2012robust} proved that under some conditions, the solution of \eqref{2} is unique and  is one of the solutions to \eqref{1}. This solution $Z^\star$ can be transformed to obtain an affinity matrix for data $X$, which can then be used for subspace segmentation. The uniqueness of \eqref{2} is given by Wei and Lin \cite{siming2011analysis}:

\begin{theorem}
Suppose the skinny SVD of $X$ is $X=\widetilde{U}\widetilde{\Sigma} \widetilde{V}'$, then the minimizer to problem \eqref{2} is
uniquely defined by \\
\begin{equation}
Z^\star=\widetilde{V}\widetilde{V}^{T}.
\end{equation}
\end{theorem}
This formula naturally implies that $Z^{\star}$ is precisely able to recover an affinity matrix in ~\cite{costeira1998multibody}.

Since the solution of \eqref{2} is  one of the solutions for \eqref{1}, we recommend referring to see corallary 4.1 in~\cite{liu2012robust}. To make the model robust to noise, Liu $et \ al$.~\cite{liu2012robust} proposed the following noisy LRR nuclear norm model:
 \begin{equation}\label{3}
 \min_{Z,E}||Z||_{\star}+\mu ||E||_{2,1},\  s.t. \ X=XZ+E,
 \end{equation}
 where $E \in R^{m\times n}$, $||E||_{2,1}=\sum_{j}\sqrt{\sum_i(E_{i,j})^2}$ is the group norm of $E$.

 In order to solve \eqref{3}, several algorithms based on SVD were designed, which are lack of speed. So Chen et al.~\cite{chen2018algorithm} put $Z$ into a low rank factorization $Z=\hat{U}V$, and they proposed the following matrix factorization model:
 \begin{equation}\label{4}
\min_{\hat{U},V,E} ||E||_{2,1} \ s.t. \ X=X\hat{U}V+E,
\end{equation}
where $\hat{U} \in R^{n\times r}$, $V \in R^{r \times n}$. If we write $U=X\hat{U}$, then the model is expressed as follows:
 \begin{equation}\label{5}
\min_{U,V,E} ||E||_{2,1} \ s.t. \ X=UV+E.
\end{equation}

However, the rank $r$ of this model must be specified. Chen $et \ al$.
\cite{chen2018algorithm} designed a greedy method to find the optimal rank:\\
1.~Provide the interval $d$ and hyperparameter $\mu$.\\
2.~Solve the problem \eqref{5}~when $r=1, d+1,\ 2d+1，...,kd+1,...$ and stop
 when $r+\mu ||E(r)||_{2,1}$ begins to worsen. Thus, they search through the options one by one to find the optimum rank.

If we assume the optimal rank $r=r^{\star}$, in this case, the solution obtained from \eqref{5} is $(U^\star,V^\star,E^\star)$. According to the fact that data space $X$ is full and the theorem in (\cite{liu2010robust, liu2012robust}), we obtain the optimal $Z^\star$ by $Z^\star=X^{+}U^\star V^\star$~($X^{+}$ is the pseudoinverse of $X$). The obtained rank is heavily dependent on the hyperparameter $\mu$, and numerous additional iterative calculations must be done before the optimal rank is obtained. In order to find the rank adaptively to reduce the number of iterations, we design a new model in section~3, which adds the group norm regularization term $||U||_{2,1}$ to the model \eqref{5}.

\section{The Group norm regularized factorization model and  algorithm }
For calculating speed, the matrix factorization method is superior to the nuclear norm approximation method. However, it is difficult to estimate the rank of the restored matrix $Z$ using the former method. So, our goal is to find an adaptive method for estimating the ranks of different types of data. The rank of a matrix is determined by the number of rows or columns in the factor matrix, and the rank of the matrix is reduced if some columns are zero. So, we first take an oversized factor matrix, and make the number of columns in the factor matrix zero by introducing group norm regularization; this achieves the purpose of adjusting the rank adaptively.

\subsection{The Group norm regularized factorization model}
If we assume that $X\in R^{m\times n}$ is a matrix of data samples, $m$ is the dimension of the data, $n$ is the number of data points, and some data contain noise. We hope to remove noise and represent clean data with a low rank to obtain an affinity matrix. We obtain the group norm regularized factorization model (GNRFM) by adding the group norm regularization term $||U||_{2,1}$ to \eqref{5}:
  \begin{equation}\label{6}
\min_{U,V,E} ||E||_{2,1}+\mu_U ||U||_{2,1}+ \frac{\mu_V}{2} ||V||_F^2 \ s.t. \ X=UV+E,
\end{equation}
where $X \in R^{m \times n}$, $U \in R^{m \times K}$, $V \in R^{K \times n}$, $K$ is a larger number, and $||.||_F$ is classic Frobenius norm. $||U||_{2,1}$ is the group norm of $U$, and $||U||_{2,1}=\sum_{j}\sqrt{\sum_i(U_{i,j})^2}$.
The true rank $r$ of $X$ is usually unknown, and $K$ is a relatively large initial guess, such as $K=n$.~Owing to the group norm regularization, some columns in U will be equal to zero under proper parameters $\mu_U$, $\mu_V$. Assuming $s$ columns in $U$ will equal zero based on the group norm $||.||_{2,1}$, then we can get $rank(UV)\leq K-s$. So, we reach the goal of adjusting the rank of $UV$ adaptively only by introducing group norm regularization.  $||V||_F^2$ is also very important because $U$ and $V$ play roles of balance and mutual restraint in GNRFM.

In summary, the GNRFM model adaptively estimates ranks for different types of data without the need to design additional updated rank strategies. The regularization terms make the model more resistant to noise. Of course, we introduce two extra hyperparameters $\mu_U$ and $\mu_V$, but numerical results show that our model is less sensitive to hyperparameters relative to other models.
\subsection{The Augmented Lagrangian Method~(ALM) }
In this section, we introduce the ALM method to solve \eqref{6}. For such  bi-convex problems, i.e, convex in U for a fixed V and convex in V for a fixed U, Sun \cite{sun2014alternating}, Shen \cite{shen2014augmented}, Xu \cite{xu2012alternating}, and Chen \cite{chen2018algorithm} all used similar ALM methods to solve such bi-convex problems, and obtained relatively good numerical results. The augmented Lagrange function in formula \eqref{6} is defined as follows :
\begin{eqnarray}\label{7}
L_{\beta}(U,V,E,Y)&= &||E||_{2,1}+\mu_U ||U||_{2,1}+ \frac{\mu_V}{2} ||V||_F^2 + <Y,UV+E-X>+ \frac{\beta}{2}||UV+E-X||_F^2, \notag \\
 &= & ||E||_{2,1}+\mu_U ||U||_{2,1}+ \frac{\mu_V}{2} ||V||_F^2  +  \frac{\beta}{2}||UV+E-X+\frac{Y}{\beta}||_F^2-\frac{||Y||_F^2}{2\beta}, 
\end{eqnarray}
where $\beta$ is a penalty parameter, $Y \in R^{m\times n}$ is the Lagrange multiplier corresponding to the constraint $X=UV+E$ and $<,> $ is the usual inner product.

 It is well-known that, starting from $Y_0=0$, the classic Augmented Lagrangian Method solves
\begin{equation} \label{2.4}
\min_{U,V,E}L_{\beta_t}(U,V,E,Y_t),
\end{equation}
at the $t$-th iteration
and then updates $Y_{t+1}=Y_{t}+\beta_t(U_{t+1}V_{t+1}+E_{t+1}-X)$. Similar to classical ALM, we  update $E$ and $( U,V )$ at the $t$-th iteration separately:
\begin{subequations}\label{8}
\begin{align}
(U_{t+1},V_{t+1})=&\mathop{argmin}\limits_{U,V}~\mu_U ||U||_{2,1}+ \frac{\mu_V}{2} ||V||_F^2  + \frac{\beta_t}{2}||UV+E_t-X+\frac{Y_t}{\beta_t}||_F^2,\\
E_{t+1}=&\mathop{argmin}\limits_{E}~ ||E||_{2,1} + \frac{\beta_t}{2}||U_{t+1}V_{t+1}+E-X  +\frac{Y_t}{\beta_t}||_F^2.
\end{align}
\end{subequations}

It is difficult to solve (10a) directly because $U$ and $V$ are coupled, so we propose a method called the inner iteration technique to obtain an approximate solution:
\begin{equation}
U^{l+1}_{t}=\mathop{argmin}\limits_{U}~\mu_U ||U||_{2,1} + \frac{\beta_t}{2}||UV^{l}_{t}+E_t-X+\frac{Y_t}{\beta_t}||_F^2,
\end{equation}
\begin{equation}\label{9}
V^{l+1}_{t}=\mathop{argmin}\limits_{V}~\frac{\mu_V}{2} ||V||_F^2 + \frac{\beta_t}{2}||U^{l}_{t}V+E_t-X+\frac{Y_t}{\beta_t}||_F^2,
\end{equation}
where $l$ represents the inner iteration steps. At this point, $V$ is solved by least square method:
\begin{equation}\label{10}
V^{l+1}_{t}=(\mu_V I_K+\beta_t{U_t^l}^TU_t^l)^{-1}\beta_t{U_t^l}^T(X-E_t-\frac{Y_t}{\beta_t}),
\end{equation}
where $I_K$ is a $K$-order identity matrix.

Since $U$ is difficult to solve, inspired by~\cite{lin2011linearized}, we conduct quadratic linearizing in (11) and add a proximal term:
\begin{equation}\label{11}
U^{l+1}_{t}=\mathop{argmin}\limits_{U}~\mu_U ||U||_{2,1} +\frac{\beta_t*\xi(V_t^l)}{2}||U-U_t^l +(U_t^lV^{l}_{t}+E_t-X+\frac{Y_t}{\beta_t}){V_t^l}^T/\xi(V_t^l)||_F^2,
\end{equation}
where $\xi(V_t^l)=1.02\sigma^2_{max}(V_t^l)$ is the same as proposed in \cite{lin2011linearized}.

  We obtain the solution to \eqref{11} by soft threshold shrinkage:
\begin{equation} \label{12}
(U_{t}^{l+1})_{*i}=max(||(Q_{*i})||_2-\frac{
\mu_U}{\beta_t*\xi(V_t^l)},0)\frac{Q_{*i}}{||Q_{*i}||_2},
\end{equation}
where $Q_{*i}=(U_t^l-(U_t^lV^{l}_{t}+E_t-X+\frac{Y_t}{\beta_t}){V_t^l}^T/\xi(V_t^l))_{*i}$, $i=1,2...,K$, and $X_{*i}$ represents the $i$-th column of $X$.

Owing to the soft-thresholding rule, some columns in $U$ are equal to zero, so we obtain a low-rank solution. Similarly, we get an explicit expression of $E$:
 \begin{equation} \label{13}
 (E_{t+1})_{*i}= max(||(X-U_{t+1}V_{t+1}-Y_t/\beta_t)_{*i}||_2-\frac{1}{\beta_t},0)\times \frac{(X-U_{t+1}V_{t+1}-Y_t/\beta_t)_{*i}}{||(X-U_{t+1}V_{t+1}-Y_t/\beta_t)_{*i}||_2},
 \end{equation}
where $i=1,2...,n.$

 To avoid ALM converging to an infeasible point, we adopt the strategies  proposed by Lu and Zhang \cite{lu2012augmented} to update $\beta_t$ in the third part of Algorithm~1. At this point, we have given the explicit formula for updating the variables in \eqref{2.4} at the $t$-th iteration.
According to the above updating formula, we employ Algorithm~1 to solve problem \eqref{6}.

\begin{algorithm}[htb]
\setstretch{1}
\caption{: ALM for GNRFM}
\label{alg.Framwork}
\begin{algorithmic}
\REQUIRE  Data $X$, an overestimated rank $K $, and hypermeters $\mu_U, \mu_V$; \\
\ENSURE $U_0 \in R^{m \times K}$, $V_0 \in R^{K \times n}$ are obtained by the SVD of X ($U_0=\widetilde{U}, V_0=\widetilde{\Sigma}\widetilde{V}$),  $Z_0=Y_0=\mathbf{0}$, $\beta_0$, $\rho>1$,
$\zeta \in (0,1)$, $\nu \in (0,1)$, $\{\epsilon_t\}\geq0 $, with $\lim\limits_{t\rightarrow \infty}\epsilon_t =0$, and a sufficiently large constant $T>\max\{f(x),L_{\beta(x,Y)}\}$ and $t=0$;
\WHILE{not convergent}
\STATE 1. With $Y=Y_t$, and $\beta=\beta_t$, compute $V_{t+1},U_{t+1}$ according to \eqref{10} and \eqref{12} to find an  approximate solution for \eqref{8}. Then, update $E_{t+1}$  according to \eqref{13}, so we can find an approximate point $x_{t+1}=(U_{t+1},V_{t+1},E_{t+1})$~s.t.
\begin{equation}
||\nabla_xL_{\beta_t}(x,Y_t)||\leq \epsilon_{t+1},~~~~~~~~~ L_{\beta_t}(x_{t+1},Y_t)<T;
\end{equation}
\STATE 2. Set
\begin{equation}
 Y_{t+1}=Y_{t}+\beta_t(U_{t+1}V_{t+1}+E_{t+1}-X);
 \end{equation}
 3. If $t>0$ and
 \begin{equation}
 ||U_{t+1}V_{t+1}+E_{t+1}-X||\leq \zeta ||U_tV_t+E_t-X||,
 \end{equation}
 then set $\beta_{t+1}=\beta_t.$ Otherwise, set
 \begin{equation}
 \beta_{t+1}=\max\{\rho \beta_t,||Y_{t+1}||^{1+\nu}\};
 \end{equation}
\STATE 4. Set $t\leftarrow t+1$.
%and check the stop citeria:$||P_\Omega(U_t'Z_t-M)||_F^2/||M||_F^2<tol$
\ENDWHILE
\end{algorithmic}
\end{algorithm}

Many books or articles (Boyd\cite{boyd2011distributed}, Chen \cite{chen2018algorithm}, Sun \cite{sun2014alternating}, Shen \cite{shen2014augmented}, Xu \cite{xu2012alternating}) all numerically show strong convergence behaviour and fast calculation speed for non-convex problems like this type of matrix factorization. But the proof of convergence about applying the ALM to non-convex problems is still a very difficult matter at present. The last four articles assume that ALM algorithm converges to the KKT point under some strong conditions, which are difficult to verify theoretically. Thus this topic is deserving of future research. For a detailed discussion of convergence, we recommend readers to see \cite{chen2018algorithm}.

%%%%%%%%%%%%%%%%%%%%%%%%%%%%%%%%%%%%%%%%%%%%%%%%%%%%%%%%%%%%%%%%%%%%%%%%%%%%%%%%%5
\subsection {The Accelerated ALM Method~(AALM) for GNRFM}
%%%%%%%%%%%%%%%%%%%%%%%%%%%%%%%%%%%%%%%%%%%%%%%%%%%%%%%%%%%%%%%%%%%%%%%%%%%%%%%%%%%%
In this section, we propose two techniques to accelerate the ALM for GNRFM. The techniques aim to reduce the computational complexity at each iteration and the number of iterations. In Figure 1 and Figure 2 of Section 4, we compare the accelerated and unaccelerated ALM on synthetic data.

The first technique conducts inner iterations during only one step for $U$ and $V$ , which is also adopted in \cite{chen2018algorithm}:
\begin{equation}\label{14}
V_{t+1}=(\mu_V I_k+\beta_t{U_t}^TU_t)^{-1}\beta_t{U_t}^T(X-E_t-\frac{Y_t}{\beta_t}),
\end{equation}
\begin{equation} \label{15}
(U_{t+1})_{*i}=max(||(Q_{*i})||_2-\frac{\mu_U}{\beta_t*\xi(V_t)},0)\frac{Q_{*i}}{||Q_{*i}||_2},
\end{equation}
where $Q_{*i}=(U_t-(U_tV_{t}+E_t-X+\frac{Y_t}{\beta_t}){V^T_{t+1}}/\xi(V_t))_{*i}, \ i=1,2...,K.$ Only one $V_{t+1}$ is replaced here to facilitate the later proof. Although we solve (10a) and (U, V) at the same time with only one inner iteration, the numerical values show that the Algorithm 2 converges in about ten steps.

The computational complexity primarily stems from the matrix multiplication at each iteration.
In the present case, some columns from matrix $U$ are zero owing to the utilization of  group norm regularization. This fact inspires the second technique, that is, we delete the zero columns in $U$ and the corresponding rows in $V$  before performing matrix multiplication.  In  numerical experiments, we find that $r\leq K_{t+1} \leq K_t \leq K$. Here, $K_t$ is the number of non-zero columns of $U$ at the $t$-th iteration. Next, we offer a theoretical proof to ensure that deleting the zero vector column does not affect convergence.
\begin{theorem}
When updating $U,V$ by \eqref{15} and \eqref{14}, if the $i$-th column in $U$ is a zero vector column, then this column is always a zero vector in subsequent iterations.
\end{theorem}
\begin{proof}
According to \eqref{14}, we get:
\begin{equation*}
 \mu_V V_{t+1}=\beta_t{U_t}^T(X-E_t-\frac{Y_t}{\beta_t}-U_tV_{t+1}),
\end{equation*}
If $(U_t)_{\star i}=0$, then $(V_{t+1})_{i\star}=\mathbf{0}$. Based on \eqref{15}, we have  $Q_{\star i}=\mathbf{0}$, so $(U_{t+1})_{\star i}=\mathbf{0}$.
\end{proof}

Therefore, the second technique does not affect the convergence, and it speeds up the calculation. We provide a detailed comparison of ALM and AALM in terms of synthetic data in the next section. By applying the above acceleration techniques, we arrive at Algorithm 2 as below.

\begin{algorithm}[htb]
\setstretch{1}
\caption{Accelerated ALM (AALM) for GNRFM}
\label{alg.Framwork}
\begin{algorithmic}[t]
\REQUIRE Data $X$, an overestimated rank $K $, and hypermeters $\mu_U, \mu_V$;  \\
\ENSURE $U_0 \in R^{m \times K}$, $V_0 \in R^{K \times n}$ are obtained by SVD of X ($U_0=\widetilde{U}, V_0=\widetilde{\Sigma}\widetilde{V}$),  $Z_0=Y_0=\mathbf{0}$, $\beta_0$, $\rho>1$, $\zeta \in (0,1)$, $\nu \in (0,1)$;
\WHILE{not convergent}
\STATE 1. With $Y=Y_t$, and $\beta=\beta_t$, compute $U_{t+1},V_{t+1}$ according to \eqref{14} and \eqref{15} to find an approximate solution for \eqref{8}. Delete zero columns in $U$ and corresponding rows of $V$. Then, update $E_{t+1}$  according to \eqref{13};
\STATE 2. Set
\begin{equation}
 Y_{t+1}=Y_{t}+\beta_t(U_{t+1}V_{t+1}+E_{t+1}-X); \notag
 \end{equation}
 3. If $t>0$ and
 \begin{equation}
 ||U_{t+1}V_{t+1}+E_{t+1}-X||\leq \zeta ||U_tV_t+E_t-X||, \notag
 \end{equation}
 then set $\beta_{t+1}=\beta_t.$ Otherwise,set
 \begin{equation}
 \beta_{t+1}=\max\{\rho \beta_t,||Y_{t+1}||^{1+\nu}\}; \notag
 \end{equation}
\STATE 4. Set $t\leftarrow t+1$.
\ENDWHILE
\end{algorithmic}
\end{algorithm}

\subsection{Time complexity }
The time complexity for AALM algorithms depends on two factors: the total number of iterations and the computational complexity of each iteration. The numerical results show that our algorithm converges rapidly and only needs about ten iterations. So, we focus on discussing the computational cost per iteration.

From Algorithm 2, we see that the computational complexity arises form matrix multiplication. At the $t$-th iteration, the computational complexity of AALM is $O(K_t^2n+K_t^2m+K_tmn)$. Here, $K_t$ is the number of non-zero vector columns in $U$ at the $t$-th iteration. Since $r\leq K_{t+1} \leq K_t \leq K$, the time complexity per iteration for AALM is $O(rmn)$.
\subsection{Subspace Segmentation~(Clustering)}
Similar to Liu~\cite{liu2012robust}, we design the following algorithm to perform subspace segmentation~(clustering) based on  $U^\star, V^\star$ obtained by solving \eqref{6}.
\begin{algorithm}[htb]
\setstretch{1}
\caption{Subspace Segmentation~(Clustering)}
\label{alg.Framwork}
\begin{algorithmic}[t]
\REQUIRE  Data $X$, an overestimated rank $K $, hypermeters $\mu_U, \mu_V$, and number $\tilde{k}$ of subspaces; \\
\STATE 1. Obtain the minimizer $U^\star,V^\star$ by Algorithm 2;
\STATE 2. Compute $Z^\star = X^{+}U^\star V^\star$;
\STATE 3. Compute the skinny SVD: $Z^\star = \widehat{U}\widehat{\Sigma}\widehat{V}^T$;
\STATE 4. $\overline{U}_{i*}=(\widehat{U}\widehat{\Sigma}^\frac{1}{2})_{i*}/||(\widehat{U}\widehat{\Sigma}^\frac{1}{2})_{i*}||_2, \ \ i=1,2,...,m;$
\STATE 5. Obtain an affinity matrix $ W=(w_{i,j})=([\overline{U} \overline{U}^T]_{i,j}^2)$;
\STATE 6 Use $W$ to perform $N\_Cut$ and segment the data samples into $\tilde{k}$ clusters.
\end{algorithmic}
\end{algorithm}

In the fifth step of Algorithm 3, each item is squared to ensure that the elements in the affinity matrix are positive. In the third and sixth steps,  one SVD decomposition is needed. For small data sets, this does not take too much time. For large-scale data, the Nystr$\ddot{o}$m approximation is also a popular family of methods to replace SVD\cite{Li2015Large}, especially for Spectral Clustering\cite{he2019fast}. The data set tested in this article is not particularly large, so the calculation of these two parts is not as critical. In summary, Algorithm 3 describes how to use the solution obtained by GNRFM for clustering.
\section{Numerical experiments}
In this section, we test the efficiency of our algorithm and compare it with other algorithms. We implement our algorithm on a PC with  3.2GHZ  AMD Ryzen 7 2700 Processor and 16GB of running memory. All computations are done in Matlab 2016b and few tasks are written with C++. We compare our algorithm with three methods:~ LADMAP(A)~\cite{lin2011linearized}, IRLS~\cite{lu2014smoothed}, and HMFALM~\cite{chen2018algorithm}.
The first method is based on model \eqref{3}, which is faster than other SVD algorithms because it uses an adaptive adjustment penalty term to accelerate convergence and uses skinny SVD instead of traditional SVD. The first method also reduces the complexity from $O(n^3)$ to $O(rn^2)$, where $r$ is the predicted rank of $Z$. IRLS smoothes the objective function by introducing regular terms and then uses the weighted least squares method to solve the variables alternately. Although the singular value decomposition is not required during the algorithm, the matrix product complexity is still $O(n^3)$. During the solution process, the Matlab command lyap is used to solve the Sylvester equation (sometimes the solution of the equation is not unique, and the program is terminated), but in some problems, the number of iterations is less than that of LADMPA(A). HMFALM, which is based on the matrix factorization model \eqref{5}, does not require calculating SVD; so as to be $O(rmn)$ complexity, where $m$ is the dimension of the data. The outer loop finds the rank, which starts from 1 and increases by step $d$. Under each outer loop, the inner loop is calculated iteratively until the stop condition is met to break out of the inner loop, and until the best rank interval is obtained to find the optimal $r$ one by one. HMFALM is faster than the first two algorithms, but it is very sensitive to the hyperparameter $\mu$, and the anti-noise ability is not good without a regularization term.

Our model GNRFM adds the group norm regularization term to the matrix factorization model \eqref{5} and uses the nature of the group norm regularization term; the factor matrix has zero columns so as to adaptively reduce the rank. Although our rank starts to decrease from a large number $K$, it only takes a few steps to iterate from a large rank to a small rank. The numerical results show that our AALM algorithm converges in about ten iterations for  \eqref{6}.
The stopping criteria in our numerical experiments is defined as follows:
\begin{eqnarray}
\centering
||U_tV_t+E_t-X||_F/||X||_F<\varepsilon,
\end{eqnarray}
where $\varepsilon$ is a moderately small number. In the following numerical experiments, we adopt the classic evaluation index indicators, Accuracy~(Acc) and Normalized Mutual Information~(NMI), to measure the clustering results. The larger the Acc and NMI values are, the better the clustering performance is.

\subsection{Experiments on synthetic data}
We first compare the ALM and AALM (before and after acceleration) on the synthetic data. For the inner iteration of ALM, we try two stopping criteria:\\
1. The inner iteration stops in five fixed steps.\\
2. The stop criterion of the inner iteration converges when $|||U_t^{l+1}V_t^{l+1}||_F/||U_t^{l}V_t^{l}||_F-1|<\varepsilon$.

The construction method for noisy synthetic data is the same as in \cite{lin2011linearized,liu2010robust,xiao2015falrr,chen2018algorithm}. The specific construction procedure is as follows.
First, we denote the number of subspaces as $s$, and the number of bases in each subspace as $\widetilde{r}$, while the dimension of the data is $\widetilde{d}$. For the first subspace, we construct the basis $B_1$, which is a random orthogonal matrix with the dimension $\widetilde{d}\times \widetilde{r}$. Basis $\{B_i\}^s_{i=2}$ in the corresponding subspace is obtained by $B_{i+1}=TB_i$, where $T$ is a random rotation matrix. This ensures that these subspaces are independent of each other, and the basis in each subspace is linear independent. Then, in the $i$-th space, we use the basis to generate $p$ samples: $X_i=B_iP_i$, where $P_i\in R^{n\times p}$ is independent and identically distributed, obeying the standard normal distribution $N(0,1)$. Then, we randomly contaminate 20$\%$ of the data, such as the data vector $x$, by adding noise according to the following formula:
\begin{equation}
x=x+\sigma||x||_2\times \eta_{\widetilde{d}\times1},
\end{equation}
where $\eta$ is a zero mean unit variance Gaussian noise vector. Finally, we get the data matrix $X=[X_1,X_2,...,X_s]\in R^{\widetilde{d}\times sp}$.

%, and Figure 1 shows that the effect after acceleration is better than that without acceleration

We denote $s=40, p=50, \widetilde{d}=2000, \widetilde{r}=5$, and $\sigma=0.05$ generate synthetic data as described above. In Figure 1 and Figure 2, ALM and AALM are compared. In Figure 1, the horizontal axis represents time(s), and the vertical axis is obtained after $log_{10}$ transformation of  error $||U_tV_t+E_t-X||_F/||X||_F$. In Figure 2, the vertical axis represents the relative error of $E_0$, which is noise added into the synthetic data. Figure 1 compares the convergence of ALM and AALM, and Figure 2 compares the results of ALM and AALM algorithms to capture noise. The purple line shows the criterion for the inner iteration of ALM, which then adopts the second criterion: each step iterates until the inner iteration converges. The green line represents the inner iteration with five fixed steps. The red line illustrates the inner iteration with a single fixed step, but the zero vector columns are not deleted. The blue line is the inner iteration with one step and deletes the zero vector columns of each $U$. From Figure 1 and Figure 2, we see that when AALM uses the two acceleration techniques, it converges fastest and obtains the  best recovery result. When ALM converges within each inner iteration, it requires the fewest outer iteration steps~(11 steps), but it is the slowest. Comparing the blue line with the red line, we observe that deleting the zero vector columns validates our previous analysis; with no effect on the convergence and result, these deletions save memory space and speed up the calculations. From Figures 1 and 2, we see that the inner iteration does not need to converge; even if one step is adopted, this greatly reduces the calculation time.
\begin{figure}[htb]
    \centering
    \includegraphics[width = 8cm,height=7.5cm]{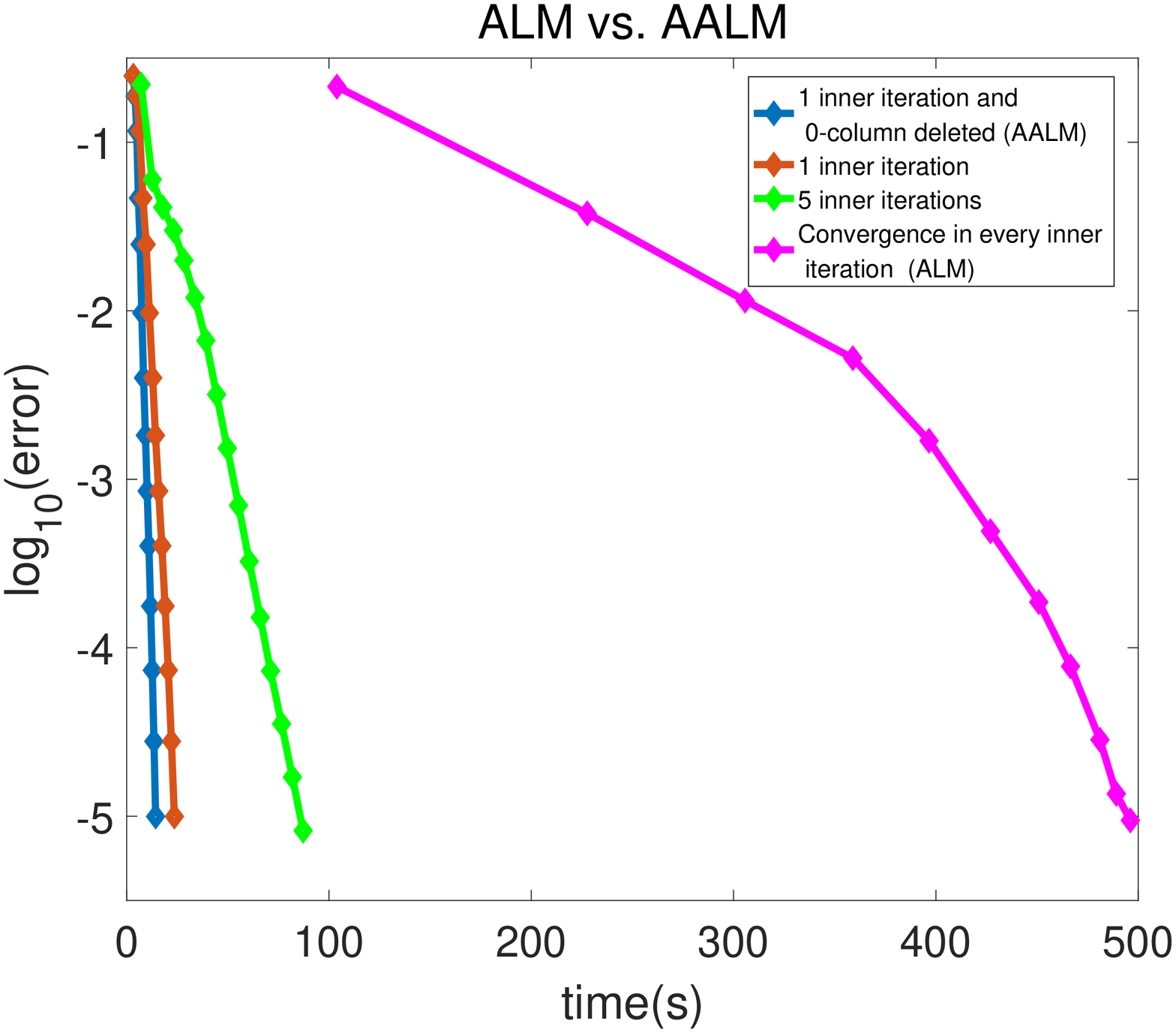}
    \centering
    \caption{Comparison of the ALM and AALM  convergences on synthetic data}
    %\notag  %  用于交叉引用
\end{figure}

\begin{figure}[htb]
    \centering
    \includegraphics[width = 8cm,height=7.5cm]{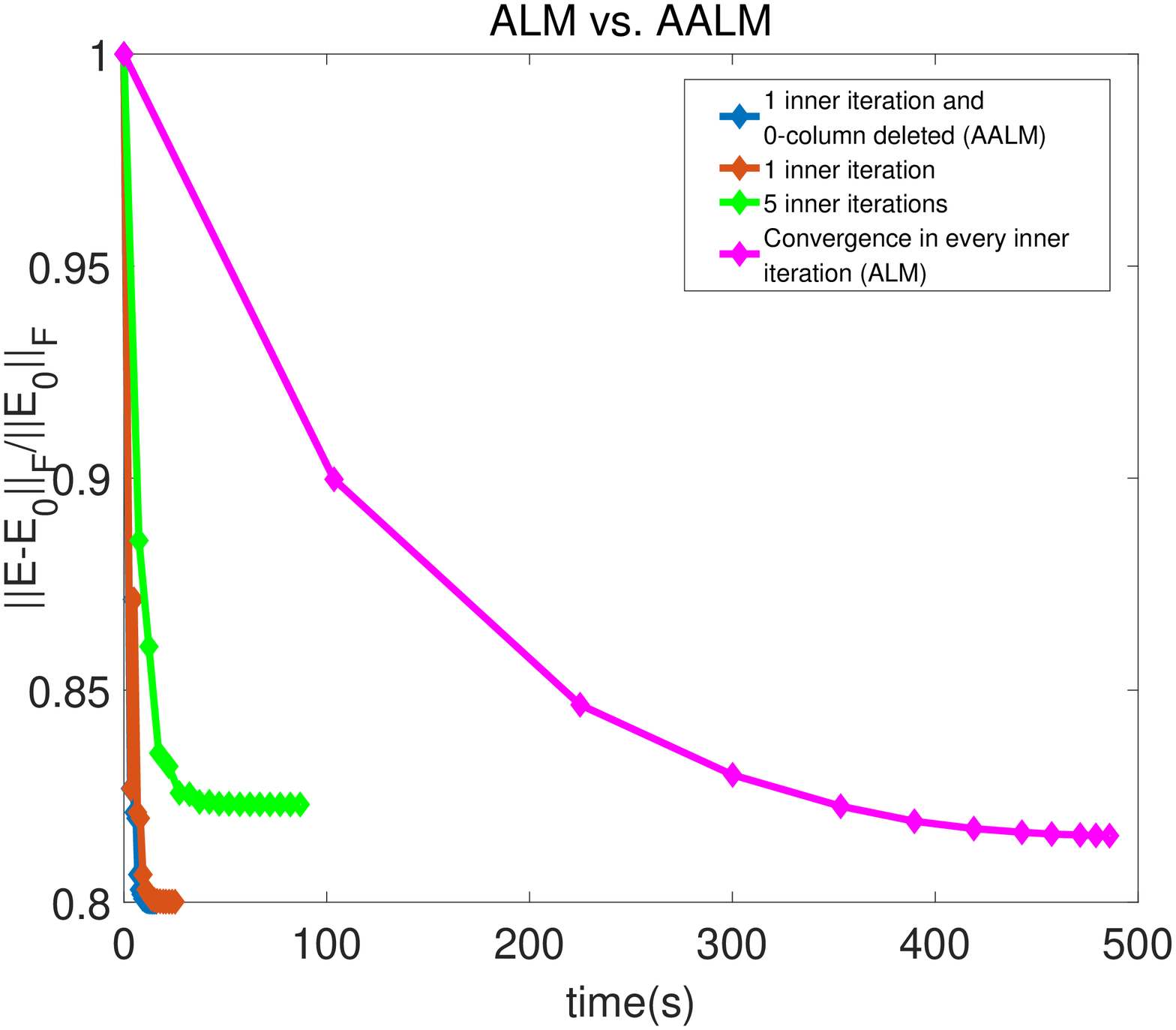}
    \centering
    \caption{Comparison of the results of ALM and AALM on synthetic data}
    %\notag  %  用于交叉引用
\end{figure}

 From Table 2 to Table 4, we use LADMAP(A), HMFALM, and AALM separately to obtain the corresponding affinity matrix on synthetic data with different levels of noise, and then Algorithm 3 is adopted to perform clustering. The goal is to verify noise resistance and sensitivity to hyperparameters between the AALM and several other algorithms for comparison. If the noise level is too high, then the information from the data is lost. So, for the intensity of the noise, we select $\sigma=[0.05,0.1,0.2]$, that is 5\%, 10\%, 20\% of the original data. For the selection of hyperparameters, we select $\mu=[0.1,0.2,0.5]$ for the three algorithms LADMAP(A), IRLS, HMFALM (the parameters mentioned by these three articles). Paramter $\mu$ for LADMAP(A) and IRLS is the regularized parameter for resisting noise, From model \eqref{3}, it is apparent that the larger the $\mu$, the better the anti-noise effect. For the HMFALM model, parameter $\mu$ plays a role in finding ranks. With respect to our algorithm, $ (\mu_U=1,\mu_V=10 ) $, $ (\mu_U=1,\mu_V=20) $, and $ (\mu_U=1,\mu_V=50) $ are selected. $\mu_U$ and $\mu_V$ are regularized parameters for resisting noise, and $\mu_U$ leads the column-sparse factor matrix $U$ to adaptively find ranks. The larger the values of $\mu_U$ and $\mu_V$, the better the noise resistance. The greater the value of $\mu_U$, the faster the rank drop. For the other parameters for IRLS and LADMAP(A) algorithms, we select optimal parameters based on the corresponding article, and we select $\varepsilon =10^{-5}$, $\beta_1=1$, $\beta_{\max}=10^5$, and $\rho=2$ for the HMFALM algorithm with the searching gap $d=0.025*n$. We select $\varepsilon =10^{-5}$, $\beta_1=1$, $\beta_{\max}=10^5$, and $\rho=2$ as the other parameters in our algorithm. $\varepsilon$ is the parameter for stopping criteria, and $\beta$ is an increased penalty parameter for the AALM algorithm. For all results, we run the algorithm three times, and take the average as the result for each synthetic data point. We illustrate the best results for each case in bold font.
\begin{table*}[!htbp]
 \caption{ Numerical results on synthetic data ($\sigma=0.05$) }
 \footnotesize
\setlength{\tabcolsep}{1pt}
 \centering
  %  \normalsize
 \begin{tabular}{p{2cm}{r}{r}{r}{r}{r}{r}{r}{r}{r}{r}{r}{r}{r}}
  \toprule[1.5pt]
             &         &\multicolumn{4}{c}{$ (\mu=0.1, \mu_U=1, \mu_V=10)$} &\multicolumn{4}{c}{$( \mu=0.2, \mu_U=1, \mu_V=20)$} &\multicolumn{4}{c}{$ (\mu=0.5, \mu_U=1, \mu_V=50)$}\\
  \toprule[1.5pt]
($s,p,\widetilde{d},\widetilde{r}$) & Method &Time(s) &Ite &Acc(\%) &NMI &Time(s) &Ite &Acc(\%) &NMI &Time(s) &Ite &Acc(\%) &NMI\\
  \midrule[1.5pt]
(10,20,200,5) &HMFALM& 0.0536 &	62 &  26.83 & 0.2647 & 0.2027 & 196 & \textbf{100.00} & \textbf{1.0000} & 0.2757 & 224	& 98.83 &	0.9798 \\
              &LADMAP(A)& 0.4703 &	48 &  98.00 & 0.9654 & 3.5960	& 326 & \textbf{100.00} & \textbf{1.0000} & 14.384 &1209 & 90.00 & 0.8603 \\
              &AALM& \textbf{0.0370} & 	\textbf{11} & \textbf{100.00} & \textbf{1.0000} & \textbf{0.0333} & \textbf{10} & \textbf{100.00}  & \textbf{1.0000} &  \textbf{0.0310} & \textbf{9}	&\textbf{100.00} & \textbf{1.0000} \\
             % & IRLS &49.607 &429 &99.00 & & &\\

(15,20,200,5) &HMFALM& 0.1193 &	75 & 46.33  & 0.4981 & 0.4767 & 191 & 100.00 & 1.0000 & 1.1207 & 259  &83.44 &	0.7793 \\
              &LADMAP(A)&0.8577 &	49 & 99.00  & 0.9847 & 6.1437	& 273 & 100.00 & 1.0000 & 29.571 & 1149 &84.22 & 0.8187 \\
              &AALM& \textbf{0.0927} &	\textbf{11} & \textbf{100.00} &\textbf{1.0000} & \textbf{0.0803} & \textbf{10}   & \textbf{100.00} &\textbf{1.0000} &  \textbf{0.0690} & \textbf{9}	 &\textbf{100.00} & \textbf{1.0000} \\
              %& IRLS &105.95 &396 &99.33 & & &\\

(20,25,500,5) &HMFALM&0.6727 & 105 & 91.93 & 0.9025 & 1.7920 & 180  & \textbf{100.00} & \textbf{1.0000} & 5.0133 & 282	 & 81.33 &0.7531\\
              &LADMAP(A)& 2.7910 &	72 & \textbf{100.00} &\textbf{1.0000} & 29.445 & 508 &  99.07 & 0.9872 &  120.57 & 1656 &91.60  &0.8848 \\
              &AALM& \textbf{0.3430} &	\textbf{11} & \textbf{100.00} &\textbf{1.0000} & \textbf{0.2953} & \textbf{10}   & \textbf{100.00} &\textbf{1.0000} & \textbf{0.2450} &\textbf{9}	 &\textbf{100.00}	& \textbf{1.0000}\\
              %& IRLS &418.76 &444 &99.87 & & &\\

(30,30,900,5) &HMFALM& 3.0670 &	123 &99.82 &0.9986	&11.048	&219 & 99.63  &0.9952 &    22.075 &269 & 80.56  &0.7604\\
              &LADMAP(A)& 16.706 &	99 &\textbf{99.96} &\textbf{0.9995} &  95.704	&446 & 88.63  &0.9277  &   292.40 &900 & 89.33   &0.8829 \\
              &AALM& \textbf{1.8723} &	\textbf{14} &98.37 &0.9807 &  \textbf{1.4730}   &\textbf{11}   &\textbf{100.00} &\textbf{1.0000} &    \textbf{1.1347} &\textbf{9}	&\textbf{100.00} &\textbf{1.0000}\\
              %&IRLS &1905.7 &422 &100.00 & & &\\

(35,40,1400,5)&HMFALM& 13.146 &	165& \textbf{100.00} &\textbf{1.0000} &53.117  & 242 & 80.55 &0.7572  &  53.511 & 242	&80.55 &	0.7589\\
              &LADMAP(A)&97.378 &	182 & 86.62 & 0.9374 &263.81 &401  &96.67 & 0.9562  & 1877.1  &1884 &90.142 &0.8744\\
              &AALM& \textbf{6.3650} &	\textbf{16} &   98.88 & 0.9854 &\textbf{5.0650} &\textbf{13}   &\textbf{99.98}  &\textbf{0.9997} &\textbf{3.7963}  &\textbf{10}	&\textbf{100.00} &\textbf{1.0000}\\
              %& IRLS &(4214.9) &(297) &(98.95) & & &\\

(40,50,2000,5) &HMFALM& 33.242 &	142 &\textbf{99.98} & \textbf{0.9998} &122.37 &230 &80.53 & 0.7541 & 123.54 &230 &	80.62 &	0.7572\\
               &LADMAP(A)&477.31 &	344 &91.43 & 0.9167	 &404.64  &266 &93.73 &0.9181 & 7957.4 &3330 &	88.17 &	0.8487\\
               &AALM&\textbf{16.472} &	\textbf{17} & 98.67 & 0.9832 & \textbf{13.770}	& \textbf{14} &\textbf{99.42} &\textbf{0.9928} &\textbf{10.161} &\textbf{10} &	\textbf{100.00} &	\textbf{1.0000}\\
              %& IRLS &null &null &null & & &\\

  \bottomrule[1.5pt]
 \end{tabular}
 \label{bs2}
\end{table*}

\begin{table*}[!htbp]
 \caption{ Numerical results on synthetic data ($\sigma=0.1$) }
 \footnotesize
\setlength{\tabcolsep}{1pt}
 \centering
  %  \normalsize
 \begin{tabular}{p{2cm}{r}{r}{r}{r}{r}{r}{r}{r}{r}{r}{r}{r}{r}}
  \toprule[1.5pt]
             &         &\multicolumn{4}{c}{$ (\mu=0.1, \mu_U=1, \mu_V=10)$} &\multicolumn{4}{c}{$( \mu=0.2, \mu_U=1, \mu_V=20)$} &\multicolumn{4}{c}{$ (\mu=0.5, \mu_U=1, \mu_V=50)$}\\
  \toprule[1.5pt]
($s,p,\widetilde{d},\widetilde{r}$) & Method &Time(s) &Ite &Acc(\%) &NMI &Time(s) &Ite &Acc(\%) &NMI &Time(s) &Ite &Acc(\%) &NMI\\
  \midrule[1.5pt]
(10,20,200,5) &HMFALM& \textbf{0.0346} &	46 & 20.67 & 0.1716 &0.1597  &166 &99.83  &0.9971 &0.4837 & 294 &	83.00 & 	0.7937\\
              &LADMAP(A)& 0.3403 &	48 & 96.50 & 0.9399 &3.0467  &283 &99.33  &0.9885 &10.223 & 728 &	84.67 &	0.8042\\
              &AALM& 0.0460 &	\textbf{11} &\textbf{99.33} &\textbf{0.9884} &\textbf{0.0387}  &\textbf{10}   &\textbf{100.00} &\textbf{1.0000} &\textbf{0.0320} &\textbf{9}   &   \textbf{99.67} &	\textbf{0.9942}\\
             % & IRLS &49.607 &429 &99.00 & & &\\

(15,20,200,5) &HMFALM& \textbf{0.0700} &	40 & 13.44 & 0.1552 &0.4567  &182 &99.78  &0.9967 &1.1640 &263 &	81.00 &	0.7436\\
              &LADMAP(A)&0.7403 &	47 & 98.45 & 0.9776 &5.1337  &229 &99.67 &0.9957 & 36.790	 &1028& 	83.00 &	0.8153\\
              &AALM& 0.1087 &	\textbf{10} &\textbf{100.00} & \textbf{1.0000} &\textbf{0.0967} &\textbf{9}     &\textbf{100.00} &\textbf{1.0000} &\textbf{0.0817} &\textbf{9}    &\textbf{99.78} &	\textbf{0.9967}\\
              %& IRLS &105.95 &396 &99.33 & & &\\

(20,25,500,5) &HMFALM&0.5487 &	89 & 40.93 & 0.4913 &2.7977 &213  &96.20 & 0.9654 &4.7137  &268 &	81.00 &	0.7515\\
              &LADMAP(A)&3.0330 &	71 & \textbf{99.93} & \textbf{0.9991} &32.993 &553  &70.40 & 0.8203 &151.55  &1664 &	82.07 &	0.7622 \\
              &AALM& \textbf{0.4077} &	\textbf{11} & 99.27 & 0.9897 &\textbf{0.3210} &\textbf{9}     &\textbf{100.00} &\textbf{1.0000} &\textbf{0.2960}  &\textbf{9}     &\textbf{100.00} &	\textbf{1.0000} \\
              %& IRLS &418.76 &444 &99.87 & & &\\

(30,30,900,5) &HMFALM& 2.4840 &	97 & 93.89 & 0.9254 &21.109 &256 &80.48  &0.7607 & 21.159	 &256  &80.59  &	0.7651 \\
              &LADMAP(A)&23.521 &	113 &93.96 & \textbf{0.9672} &130.25 &509 &89.11 &0.8645 & 823.41  &2217 &84.04 &	0.8034\\
              &AALM& \textbf{2.1263} &	\textbf{14}  &\textbf{97.00} & 0.9622 &\textbf{1.5400} &\textbf{10}   &\textbf{99.89} &\textbf{0.9986} & \textbf{1.3823}  &\textbf{9}      & \textbf{100.00} &	\textbf{1.0000}\\
              %&IRLS &1905.7 &422 &100.00 & & &\\

(35,40,1400,5)&HMFALM& 16.596 &	141 &98.76 & 0.9843 &50.914 &228 &80.57 &0.7582 & 51.119 &228   &	80.67 &	0.7593\\
              &LADMAP(A)&201.92 &	300 &65.14 & 0.7546 &1817.8 &2023 &84.79 &0.8120 &1304.0 &1233 &	85.43 &	0.8155\\
              &AALM& \textbf{6.9380} &	\textbf{17}  &\textbf{99.57}  &\textbf{0.9945} &\textbf{5.9940} &\textbf{14}    &\textbf{99.74} &\textbf{0.9965}  &\textbf{3.9077} & \textbf{9} & \textbf{100.00} & \textbf{1.0000}\\
              %& IRLS &(4214.9) &(297) &(98.95) & & &\\

(40,50,2000,5) &HMFALM& 116.01 &	216 &80.50 &0.7557 &117.16 &216  &80.65 &0.7555  &118.17  &216 &	80.43 &	0.7542\\
               &LADMAP(A)&600.83 &	391 &89.95 &0.8780 &4628.1 &2151 &82.98 &0.7851 &3577.4 &1549 &	87.18 &	0.8350\\
               &AALM&\textbf{16.258} &	\textbf{17} &\textbf{99.60} & \textbf{0.9947} &\textbf{14.914} &	\textbf{15}    &\textbf{99.62}  &\textbf{0.9949} &\textbf{11.301}  &\textbf{11}   &	\textbf{100.00} &	\textbf{1.0000}\\
              %& IRLS &null &null &null & & &\\

  \bottomrule[1.5pt]
 \end{tabular}
 \label{bs2}
\end{table*}

\begin{table*}[!htbp]
 \caption{ Numerical results on synthetic data ($\sigma=0.2$) }
 \footnotesize
\setlength{\tabcolsep}{1pt}
 \centering
  %  \normalsize
 \begin{tabular}{p{2cm}{r}{r}{r}{r}{r}{r}{r}{r}{r}{r}{r}{r}{r}}
  \toprule[1.5pt]
             &         &\multicolumn{4}{c}{$ (\mu=0.1, \mu_U=1, \mu_V=10)$} &\multicolumn{4}{c}{$( \mu=0.2, \mu_U=1, \mu_V=20)$} &\multicolumn{4}{c}{$ (\mu=0.5, \mu_U=1, \mu_V=50)$}\\
  \toprule[1.5pt]
($s,p,\widetilde{d},\widetilde{r}$) & Method &Time(s) &Ite &Acc(\%) &NMI &Time(s) &Ite &Acc(\%) &NMI &Time(s) &Ite &Acc(\%) &NMI\\
  \midrule[1.5pt]
(10,20,200,5) &HMFALM& \textbf{0.0310} &38 & 16.33 &0.1098 &0.2410 &	200 &63.33 & 0.7243 &0.4800 &283 & 82.17 & 0.7859\\
              &LADMAP(A)& 0.5753 &	69 & 85.50 &0.8355 &4.2237 &	390 &82.50 & 0.8349 &41.262 &2486 &85.50 &0.7997\\
              &AALM& 0.0470 &	\textbf{10} & \textbf{88.67} &\textbf{0.8455} &\textbf{0.0420} &	\textbf{9}    &\textbf{94.00} & \textbf{0.9081}  &\textbf{0.0380} &\textbf{9}    &\textbf{96.67}  &\textbf{0.9457}\\
             % & IRLS &49.607 &429 &99.00 & & &\\

(15,20,200,5) &HMFALM& 0.1303 &	83 &42.11  &0.4687 &0.8920 &	237 &56.00 & 0.6786 &1.1387 & 254 &81.89  &0.7575\\
              &LADMAP(A)&1.3393 &	72 &81.22  &0.8451 &9.1343 &	376 &76.00 &0.8124  &84.673 &2236 &67.11 &0.7068\\
              &AALM& \textbf{0.1147}&\textbf{10} &\textbf{91.78}  &\textbf{0.8862} &\textbf{0.0983} &	\textbf{9}     &\textbf{97.45} &\textbf{0.9623} &\textbf{0.0937}  &\textbf{9}      &\textbf{98.00}  &\textbf{0.9706}\\
              %& IRLS &105.95 &396 &99.33 & & &\\

(20,25,500,5) &HMFALM&1.0660 &	135 &79.87 &0.8076 &4.5273 &260 &81.00 &0.7492 &4.6477  &260  &80.87  &0.7487\\
              &LADMAP(A)& 21.275 &	330 &28.27 &0.4432 &52.970 &754 &32.47 &0.4836 &222.36  &2369 &80.13  &0.7472\\
              &AALM& \textbf{0.4183} &	\textbf{11}   &\textbf{87.80} &\textbf{0.8449} &\textbf{0.3400} &\textbf{9}    &\textbf{94.93} &\textbf{0.9307}  &\textbf{0.2927} &\textbf{8}       &\textbf{97.40}  &\textbf{0.9641}\\
              %& IRLS &418.76 &444 &99.87 & & &\\

(30,30,900,5) &HMFALM& 10.822 &	207 &82.22 &0.8161 &20.361 &248 &80.63 &0.7633 &20.667  &248   &80.93 &0.7694\\
              &LADMAP(A)& 104.92 &	427 &5.52 &0.0614 &830.04 &2458 &77.56 &0.8045 &1664.2 &4334 &80.59 &0.7725\\
              &AALM& \textbf{2.0727} &	\textbf{14}  &\textbf{84.52} &\textbf{0.8420} &\textbf{1.6640} &\textbf{11}   &\textbf{86.11} &\textbf{0.8597} &\textbf{1.2560} &\textbf{8}      &\textbf{96.19} &\textbf{0.9517} \\
              %&IRLS &1905.7 &422 &100.00 & & &\\

(35,40,1400,5)&HMFALM& 52.310 &	219 &80.52 &0.7586 &49.367 &219 &80.71 &0.7599 &51.278 &220  &80.60 &0.7581\\
              &LADMAP(A)&535.46 &	652 &82.93 &0.7903 &2719.2 &2652 &79.95 &0.7857 &3163.2 &2939 &81.00 &0.7682\\
              &AALM& \textbf{6.3350} &	\textbf{16}  &\textbf{91.12} &\textbf{0.8873} &\textbf{5.5310} &	\textbf{14} & \textbf{84.52} &\textbf{0.8490} &\textbf{3.3213} &	\textbf{8} & \textbf{95.31} & \textbf{0.9403}\\
              %& IRLS &(4214.9) &(297) &(98.95) & & &\\

(40,50,2000,5) &HMFALM& 118.18 &	207 &80.50 &0.7534 &114.78 &207 &80.63 &0.7558 &114.12 & 207 & 80.52 & 0.7539\\
               &LADMAP(A)&12348. &	6162 &78.12 &0.7707 &8952.7 &3885 &81.25 &0.7630 &8559.3 &3687 &79.83 &0.7691\\
               &AALM&\textbf{14.797} &	\textbf{16} &\textbf{84.90} &\textbf{0.8279} &\textbf{13.922} &	\textbf{15} &\textbf{88.17} &\textbf{0.8759} &\textbf{9.1117} &	\textbf{9} &\textbf{84.33}& \textbf{0.8506}\\
              %& IRLS &null &null &null & & &\\

  \bottomrule[1.5pt]
 \end{tabular}
 \label{bs2}
\end{table*}

\begin{figure}[htb]
    \centering
    \includegraphics[width = 8cm,height=7.5cm]{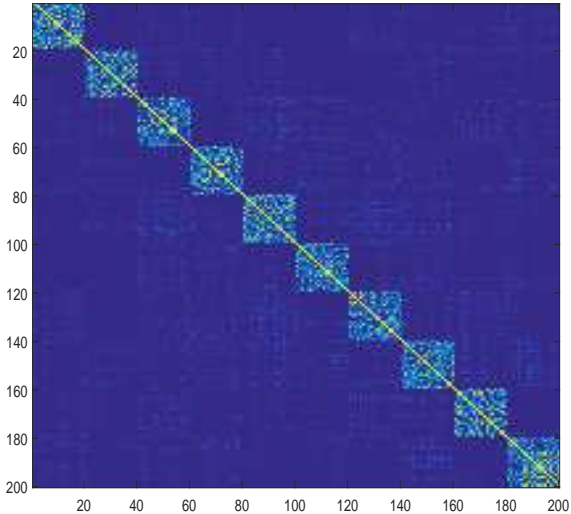}
    \centering
    \caption{The affinity matrix values for synthetic data}
    %\notag  %  用于交叉引用
\end{figure}

\begin{figure}[htb]
    \centering
    \includegraphics[width = 8cm,height=7.5cm]{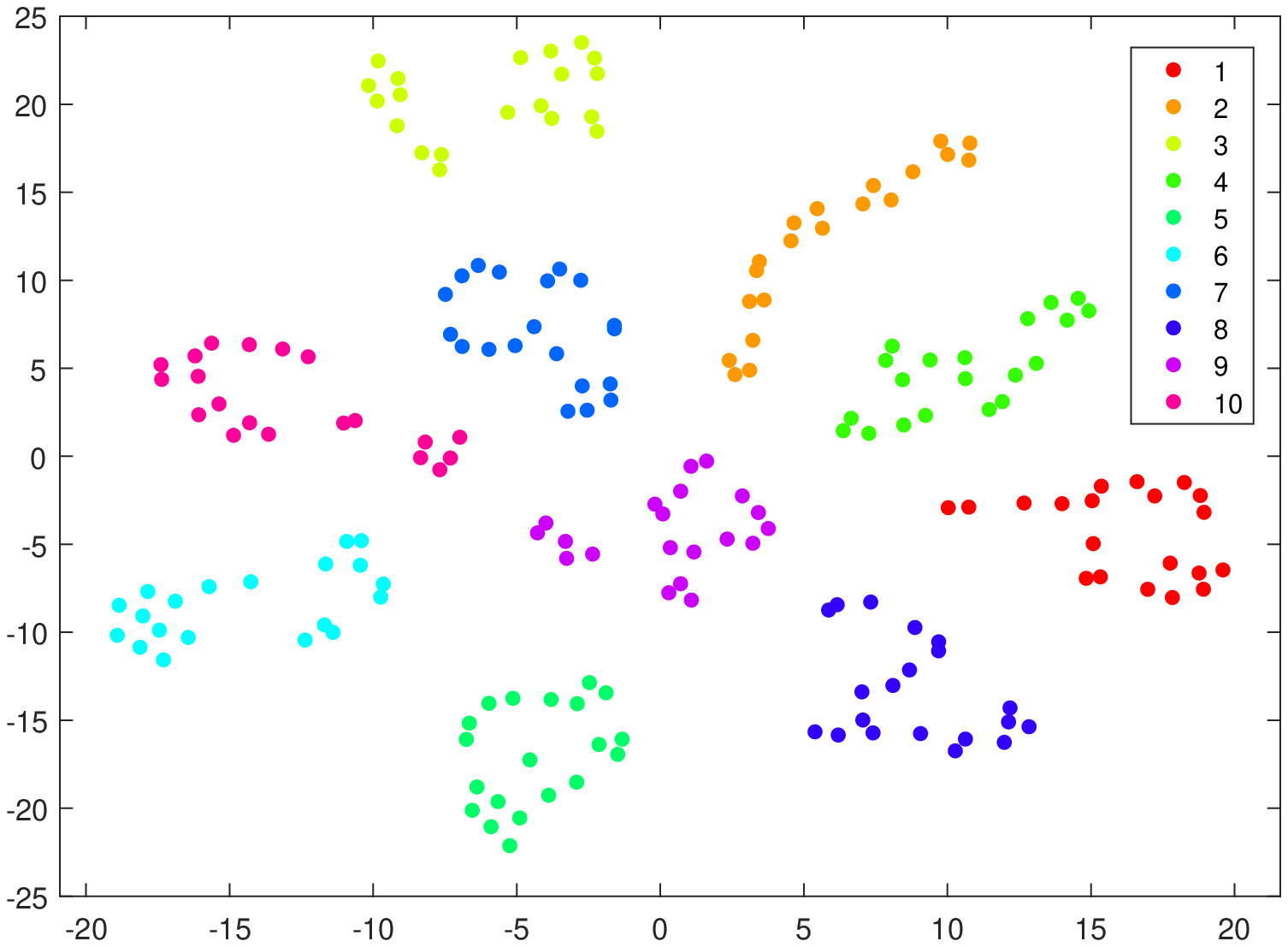}
    \centering
    \caption{Visualization of the affinity results for synthetic data using t-SNE}
    %\notag  %  用于交叉引用
\end{figure}

Table 2 to Table 4 feature synthetic data with different degrees of noise and different dimensions, and almost all of our AALM algorithm results perform the best in terms of clustering and high speed. For Table 2, which features a low-level noise situation ($\sigma=0.05$), the most accurate of the three algorithms is more than 90\% accurate for different dimensional problems, but the HMFALM and LADMAP(A) algorithms are very sensitive to parameters; that is, the results change greatly if parameters change a little, and there is no consistent parameter $\mu$ for all synthetic data. For Table 3, which features moderate noise ($\sigma=0.1$), our AALM algorithm results are still the best, with a speed increase of more than six times that of HMFALM and about forty times greater than LADMAP(A). Particularly in a high-dimensional situation $d = 2000$, our AALM is more accurate than the other two algorithms by more than 10\%. For Table 4, which has the highest noise ($\sigma=0.2$), our clustering accuracy and NMI are significantly higher than those of the other two algorithms. As a result, although our model \eqref{6} has one more hyperparameter than \eqref{5}, our model is not sensitive to hyperparameters, while the clustering result of the other two models is greatly affected by the hyperparameter $\mu$. In addition, when $\sigma=[0.05,0.1,0.2]$, our clustering accuracy is the best, and even in some cases it is upwards of 20\% more accurate than other algorithms. So, the GNRFM model, when introduced with the group norm regularization term, has good noise immunity and is robust.

Specifically, we take $(s, p, \widetilde{d}, \widetilde{r} = 10,20,200,5)$ and ($\mu_U = 1, \mu_V = 50$) from Table 2 and draw Figure 3 and Figure 4 for verification. At this time, the data come from ten subspaces; each space has 20 vectors, and each vector is 20 dimensions. Figure 3 shows the values of an affinity matrix, which is obtained by Algorithm 3 after the AALM algorithm solves the GNRFM model with color images. It can be seen that the affinity matrix is clearly divided into ten blocks, and the correlation between different subspaces is very small, and the internal correlation in the subspace is very high, the affinity matrix is well portrayed and accurately segments the subspace. In Figure 4, we use t-SNE\cite{Laurens2008tSNE} to visualize our affinity results. From Figure 4, we see that the affinity results are clearly divided into ten categories, so our models and algorithms are effective for subspace segmentation.
\subsection{Experiments on real data}
In this section, we test the clustering effectiveness of our algorithm in sport motion segmentation datasets (the Hopkins155 dataset \cite{tron2007benchmark} and HARUS dataset\cite{anguita2012human}) and face segmentation datasets (the Extended Yale B dataset \cite{georghiades2001few} and CMU PIE dataset \cite{sim2002the}). Description of the four data sets are shown in Table 5.
\begin{table*}[!htbp]
\caption{Description of real data }
\footnotesize
\setlength{\tabcolsep}{3pt}
\centering
\begin{tabular}{|c|c|c|c|c|c|c|c|}
\hline
\hline
Dataset &  Content  & Type & Subdataset & Dimension & Sample  & Number of Categories \\
\hline
\hline
Hopkins155 & Sport Motion & Videos &156 &72 &39$\sim$550 & 2 or 3 \\
\hline
Harus & Sport Motion & Sensor signals &1 & 561 &10299 & 6 \\
\hline
\multirow{2}*{Extended Yale B} & \multirow{2}*{Faces} & \multirow{2}*{Pictures} & 5 subjects & 1024 & 320 & 5 \\
\cline{4-7}& & &10 subjects &1024 &640 &10 \\
\hline\multirow{2}*{CMU PIE} & \multirow{2}*{Faces} & \multirow{2}*{Pictures} & 5 subjects & 1024 & 850 & 5 \\
\cline{4-7}& & &10 subjects &1024 &1700 &10 \\
\hline
\hline
\end{tabular}
\end{table*}

 The Hopkins155 dataset contains 156 data sequences; each video data sequence contains from 39 to 550 data vectors (from two or three motion modes), and the dimension in each data vector is 72 (24 frames $\times$ 3). We specify the number of classes (two or three classes) in each data sequence, and take advantage of HMFALM, LADM, IRLS, and AALM respectively in these 156 sequences to solve the affinity matrix and conduct clustering. In Table 4, we give the total accuracy, average NMI, average iteration steps, and average time for the data series in the conditions of two motions, three motions, and all motions. Among them, HMFALM, LADM and IRLS, we select the value of $\mu=2.4$~(the optimal parameters tested by the authors in their article). With respect to the AALM algorithm, we select $\mu_U=0.005, \mu_V=3$.
\begin{table*}[!htbp]
\caption{Comparison of motion segmentation using different algorithms on the Hopkins155 dataset }
\footnotesize
\setlength{\tabcolsep}{3pt}
\centering
\begin{tabular}{|c|c|c|c|c|c|c|c|c|c|c|c|c|c|}
\hline
\hline
\multirow{2}*{Problem} &\multicolumn{4}{|c|}{Two Motions} &\multicolumn{4}{|c|}{Three Motions} &\multicolumn{4}{|c|}{All Motions} \\ \cline{2-13}
&Time &Iter. &Acc($\%$)& NMI &Time &Iter. &Acc($\%$) &NMI &Time  &Iter. &Acc($\%$) &NMI    \\
 \hline
 \hline
HMFALM& 0.0481& 124 &96.34 &0.8794 &0.0726 &137 &95.06 &0.8902 &0.0538  &127 &95.95 &0.8819   \\ \hline
LADMAP(A)& 74.064& 28724 &96.43 &0.8884 &120.20 &37564 &95.71 &0.9103 &84.711  &30764 &96.21 &0.8935    \\ \hline
IRLS& 33.819& 189 &97.21 &\textbf{0.9053} &63.148 &182 &\textbf{95.90} &\textbf{0.9140} &40.5873  &188 &96.81 &\textbf{0.9073}    \\ \hline
AALM& \textbf{0.0224}& \textbf{14} &\textbf{97.63} &0.8979  &\textbf{0.0285} &\textbf{14} &95.17 &0.8970 &\textbf{0.0238}  &\textbf{14} &\textbf{96.87} &0.8977   \\ \hline
\hline
\end{tabular}
\end{table*}

As seen from Table 6, our algorithm is faster than the other three algorithms. For the NMI index with three motions, the results of our algorithm are not as good as IRLS, but our algorithm is much faster than  IRLS. For video data, instant clustering is very important. Overall our models and algorithms have achieved the best results.

The Human Activity Recognition Using Smartphones Dataset~(HARUS) contains sensor signals data collected by sensors with a group of 30 volunteers carrying out 6 activities. The HARUS dataset contains 10,299 signals, and each signal is a 561-dimensional feature.
In Table 7, we illustrate the accuracy, NMI, iteration steps, and time for four algorithms used on the HARUS dataset. For HMFALM, LADMAP(A), and IRLS, we select $\mu=0.01$~(the optimal parameters tested by \cite{xiao2015falrr}, because a large $\mu$  value is too time-consuming),  with respect to the AALM algorithm; we also select $\mu_U=0.005, \mu_V=3$. Because the data dimensions are too large, instead of using Algorithm 3 to convert $Z$ into an affinity matrix (SVD requires too much time), we use formula $W=(|Z|+|Z^T|)/2$ to conduct the transformation.
\begin{table}[!htbp]
\caption{Comparison of face clustering with different algorithms on the HARUS dataset}
\footnotesize
\setlength{\tabcolsep}{3pt}
\centering
\begin{tabular}{|c|c|c|c|c|}
\hline
\hline
HARUS &Time &Iter. &Acc($\%$) &NMI   \\
 \hline
 \hline
HMFALM& 81.170& 222 &55.44 & 0.5497  \\ \hline
LADMAP(A)& 15706& 7220 &59.27  &0.5220    \\ \hline
IRLS& N.A & N.A &N.A & N.A  \\ \hline
AALM& \textbf{1.3980}& \textbf{5} &\textbf{80.34}   & \textbf{0.6683}  \\ \hline
\hline
\end{tabular}
\end{table}

\begin{figure}[htb]
    \centering
    \includegraphics[width = 8cm,height=7.5cm]{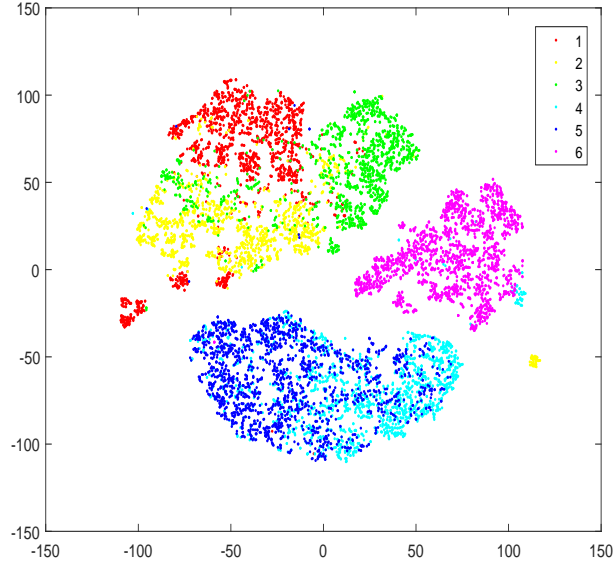}
    \centering
    \caption{Visualization of the the affinity results of HARUS using t-SNE}
    %\notag  %  用于交叉引用
\end{figure}

From Table 7, we see that for the large data set, HARUS, the IRLS algorithm does not work. In the remaining algorithms, our AALM algorithm takes less than two seconds, and it is almost 20\% more accurate than the other two algorithms. Then, we use t-SNE to visualize the affinity results in Figure 5; we see that the data is roughly separated into six categories with a little error.
\begin{figure*}[htb]
    \centering
    \includegraphics[width = 15cm,height=5.5cm]{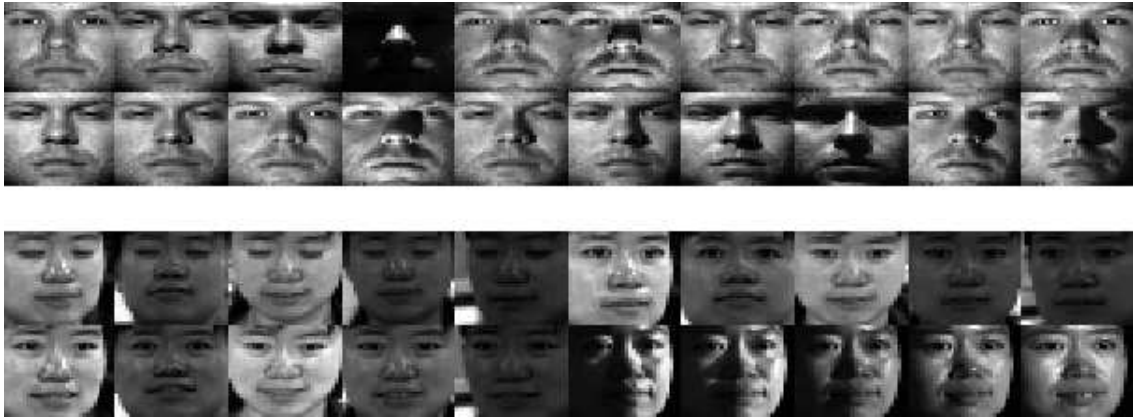}
    \centering
    \caption{Example face image from the Extended Yale B dataset and CMU PIE dataset}
    %\notag  %  用于交叉引用
\end{figure*}

In Figure 6, we show the Extended Yale B dataset and CMU PIE dataset. The Extended Yale B dataset contains 38 subjects (people), and each subject has 64 facial images. The CMU PIE dataset contains 68 subjects (people), and each subject has 170 facial images. The Extended  Yale B dataset also contains noise from different angles of light. Figure 6 above shows 20 pictures from one person's face where data has lighting noise such that some faces cannot be seen clearly or even become dark. For instance, the image in the fourth picture cannot even be identified by the human eye.
The CMU PIE dataset featrues human expressions in addition to light noise, so clustering is more difficult. Similar to Lu \cite{lu2014smoothed}, we conduct two experiments in each dataset by constructing the first five subjects and the first ten subjects into a dataset X. First, we resize all the pictures to 32$\times$ 32. Second, to reduce noise, we project them to a 30-dimensional (80) subspace for five subjects and a 60-dimensional (160) subspace for ten subjects by principle component analysis~(PCA) for the Extended Yale B (CMU PIE) dataset. Third, by applying HMFALM, LADM, IRLS, and AALM to solve the low-rank representation problem, we obtain different affinity matrices. At last, we compare the clustering results using Algorithm 3 with different affinity matrices. We set parameter $\mu=1.5$ for HMFALM, LADM, IRLS (the optimal parameters tested by the authors in their article), and $(\mu_U=1, \mu_V=20)$ for AALM.
\begin{table*}[!htbp]
\caption{Comparison of face clustering  by different algorithms on the Extended Yale B}
\footnotesize
\setlength{\tabcolsep}{3pt}
\centering
\begin{tabular}{|c|c|c|c|c|c|c|c|c|c|}
\hline
\hline
\multirow{2}*{Problem} &\multicolumn{4}{|c|}{10 subjects} &\multicolumn{4}{|c|}{5 subjects} \\ \cline{2-9}
&Time &Iter. &Acc($\%$) &NMI &Time &Iter. &Acc($\%$) &NMI   \\
 \hline
 \hline
HMFALM& 0.4130& 396 &\textbf{81.87} &\textbf{0.7680} & 0.1250& 336 &\textbf{88.44} &\textbf{0.7839} \\ \hline
LADMAP(A)& 94.121& 8429 &\textbf{81.87} &\textbf{0.7680}  & 16.872&4324 &\textbf{88.44} &\textbf{0.7839}  \\ \hline
IRLS& 86.627& 107 &81.56 &0.7658 & 14.779& 102 &88.12 &0.7800  \\ \hline
AALM& \textbf{0.0520}& \textbf{16} &\textbf{81.87} &\textbf{0.7680}   & \textbf{0.0280}& \textbf{16} &\textbf{88.44}  &\textbf{0.7839}\\ \hline
\hline
\end{tabular}
\end{table*}

\begin{figure}[htb]
    \centering
    \includegraphics[width = 8cm,height=7.5cm]{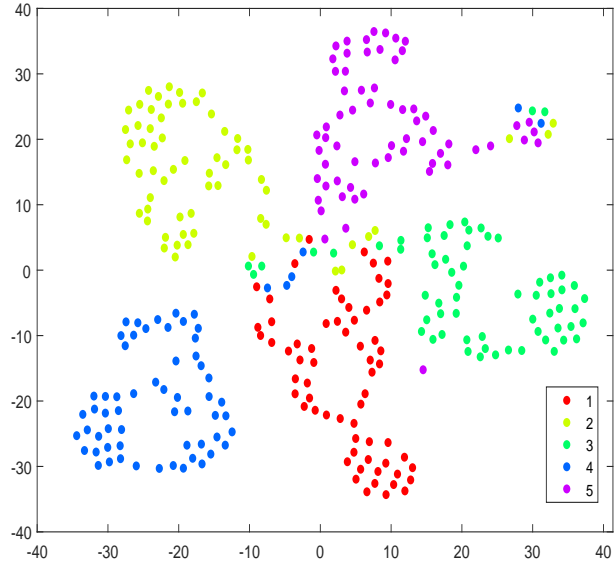}
    \centering
    \caption{Visualization of the affinity results of the Extended Yale B using t-SNE}
    %\notag  %  用于交叉引用
\end{figure}

\begin{table*}[!htbp]
\caption{Comparison of face clustering by different algorithms on the CMU PIE}
\footnotesize
\setlength{\tabcolsep}{3pt}
\centering
\begin{tabular}{|c|c|c|c|c|c|c|c|c|c|}
\hline
\hline
\multirow{2}*{Problem} &\multicolumn{4}{|c|}{10 subjects} &\multicolumn{4}{|c|}{5 subjects} \\ \cline{2-9}
&Time &Iter. &Acc($\%$) &NMI &Time &Iter. &Acc($\%$) &NMI   \\
 \hline
 \hline
HMFALM& 7.8790& 597 &44.82 &0.5550 & 1.0100& 531 &72.94 &0.8063 \\ \hline
LADMAP(A)& 5341.8& 50000 &44.82 &0.5550  & 423.98&20908 &72.94 &0.8063  \\ \hline
IRLS& 2004.5& 124 &44.82 &0.5550 & 192.80& 116 &72.94 &0.8063  \\ \hline
AALM& \textbf{0.4860}& \textbf{15} &\textbf{48.00} &\textbf{0.5871}   & \textbf{0.0910}& \textbf{16} &\textbf{73.06}  &\textbf{0.8076}\\ \hline
\hline
\end{tabular}
\end{table*}
As can be seen from Table 8, the result of the IRLS algorithm is the worst. The remaining three algorithms achieve the same accuracy and NMI, but our algorithm is the fastest. We use t-SNE to visualize the affinity results of the Extended Yale B dataset. From the five subjects in Figure 7, we see that the data is roughly separated into five categories with a little error. From Table 9, which illustrates the CMU PIE dataset, it is apparent that our algorithm achieves the best results. For the large-scale problem of ten subjects in Table 9, our algorithm is much faster than the others. Our algorithm is fast and effective for face clustering.

In summary, our algorithm achieves the best accuracy with the fastest computing speed in real problems with sports motion data (Hopkins155 dataset and HARUS dataset) and facial data (Extended Yale B dataset and CMU PIE dataset). In addition, our AALM algorithm can immediately provide the affinity results for data.
\section{Conclusion}

In this paper, we first propose the use of the group norm regularized factorization model
(GNRFM) to solve the problem of low-rank representation with minimal rank, and then we apply it to subspace segmentation. The calculation of the traditional nuclear norm approximation method is very complicated as $O(n^3)$, while the calculation of our group norm model is
$O(rmn)$. Compared with the traditional factorization model, which greedily searches for ranks, we adaptively find the rank by introducing the group norm regularization term, which
greatly reduces the number of iteration steps. In addition, the group norm regularization
term  also boasts anti-noise effects and makes the model more robust. On synthetic
data and real data, our GNRFM model and the algorithm designed for it achieve excellent clustering results. Furthermore, our AALM algorithm only requires about ten iterations, so the speed is much faster than traditional algorithms; thus, our algorithm plays a role in
the process of instant rapid clustering in the era of big data. In the future, we will also consider applying the affinity matrix obtained by the GNRFM to AP clustering \cite{frey2007clustering}, Multi-view affinity learning \cite{wang2013multi-exemplar,liu2019adaptively}, and spectral clustering with interplay manner \cite{liu2019adaptively}. In particular, our group norm regularized factorization  method can be used for a series of low-rank problems. Some theoretical proofs given by authors will be provided in the future.

\end{document}